%% file: ms.tex
\documentclass[twoside]{article}

\usepackage[accepted]{aistats2021arxiv}
\usepackage[round]{natbib}
\renewcommand\cite{\citep}
\usepackage[utf8]{inputenc} 
\usepackage[T1]{fontenc}    
\usepackage{url}            
\usepackage{booktabs}       
\usepackage{amsfonts}       
\usepackage{nicefrac}       
\usepackage{mathtools}
\usepackage{amssymb}
\usepackage{amsmath}
\usepackage{amsthm}
\usepackage{wrapfig}
\usepackage{enumitem}
\usepackage{float}

\theoremstyle{plain}
\newtheorem{proposition}{Proposition}
\newtheorem{corollary}{Corollary}
\newtheorem{remark}{Remark}
\newtheorem{theorem}{Theorem}

\theoremstyle{definition}
\newtheorem{definition}{Definition}
\newcommand{\power}{\mathit{pow}}
\newcommand{\uniform}{\mathit{Uniform}}

\begin{document}

%
\runningtitle{Pareto GAN}

%
\runningauthor{Huster, Cohen, Lin, Chan, Kamhoua, Leslie, Chiang, Sekar}

\twocolumn[

\aistatstitle{Pareto GAN: Extending the Representational Power of GANs to Heavy-Tailed Distributions}

\aistatsauthor{ Todd Huster \And Jeremy E.J. Cohen \And Zinan Lin }
\aistatsaddress{ Perspecta Labs \And Perspecta Labs \And Carnegie Mellon University}
\aistatsauthor{ Kevin Chan  \And Charles Kamhoua \And Nandi Leslie}
\aistatsaddress{ Army Research Lab  \And Army Research Lab \And Raytheon Technologies}
\aistatsauthor{ Cho-Yu Jason Chiang \And Vyas Sekar}
\aistatsaddress{ Perspecta Labs \And Carnegie Mellon University}
]

\input{abstract}
\input{intro}
\input{preliminaries}

\input{tails}

\input{loss}

\input{exp}

\input{conclusions}

\subsubsection*{Acknowledgements}
This research was sponsored by the U.S. Army Combat Capabilities Development Command Army Research Laboratory and was accomplished under Cooperative Agreement Number W911NF-13-2-0045 (ARL Cyber Security CRA). The views and conclusions contained in this document are those of the authors and should not be interpreted as representing the official policies, either expressed or implied, of the Combat Capabilities Development Command Army Research Laboratory or the U.S. Government. The U.S. Government is authorized to reproduce and distribute reprints for Government purposes notwithstanding any copyright notation here on.


\bibliography{bib}{}
\input{app}

\end{document}

%% file: abstract.tex
\begin{abstract}

Generative adversarial networks (GANs) are often billed as "universal distribution learners", but precisely what distributions they can represent and learn is still an open question.
Heavy-tailed distributions are prevalent in many different domains such as financial risk-assessment, physics, and epidemiology.
We observe that existing GAN architectures do a poor job of matching the asymptotic behavior of heavy-tailed distributions, a problem that we show stems from their construction.
Additionally, when faced with the infinite moments and large distances between outlier points that are characteristic of heavy-tailed distributions, common loss functions produce unstable or near-zero gradients.
We address these problems with the Pareto GAN. 
A Pareto GAN leverages extreme value theory and the functional properties of neural networks to learn a distribution that matches the asymptotic behavior of the marginal distributions of the features.
We identify issues with standard loss functions and propose the use of alternative metric spaces that enable stable and efficient learning.
Finally, we evaluate our proposed approach on a variety of heavy-tailed datasets.




\end{abstract}

%% file: intro.tex
\section{INTRODUCTION}

Heavy-tailed, and particularly power-law, distributions are regularly encountered in a diverse set of applications such as spectroscopy, particle motion, finance, geological processes, epidemiology, etc. \cite{michel07, fortin15, gilli06, caers99, evans04}.
Analysis of these distributions is often focused on the prevalence (i.e., risk) of rare events.
Models that can fit sample data while accurately predicting the probabilities of extreme events are useful for risk assessment in these fields.

How suitable are GANs to serve as these models? 
They have proven to be wildly successful at learning complex distributions in the image domain, \emph{without simply memorizing the data} \cite{brock18}.
They can famously generate convincing images of the faces of non-existent celebrities \cite{karras17}.
Can they also convincingly generate samples of the default rates of mortgages and the features of 100-year floods?

The universal approximation theorem \cite{hornik89} and effective loss functions like Wasserstein distance \cite{arjovsky17} suggest that GANs can learn to generate an arbitrary dataset, regardless of the distribution it was drawn from.
Of course, simple bootstrap sampling from the dataset can do the same.
It is the ability of GANs (or any generative model) to appropriately generalize from training examples that separates them from a static dataset.

In the case of heavy-tailed distributions, this generalization is not effective.
As we will show in section \ref{sec:tails}, the asymptotic behavior of a GAN marginals is predictable based solely on the combination of input distribution and activation function, irrespective of the training data.
In most cases, the generator is able to fit the training data closely and \emph{interpolate} between these points, but it does not \emph{extrapolate} in a reasonable fashion.

This does not have to be the case, however.
In their extremes, most naturally occurring marginal distributions follow one of a handful of asymptotic behaviors \cite{balkema74}.
These behaviors are all captured by the generalized Pareto distribution.
Since we are able to \emph{predict} the asymptotic behavior of a GAN from its architecture, we can therefore also \emph{design} our GAN to take on a particular belief about the tail behavior of its marginals.
To create such a generator, we feed a standard neural network a noise function with heavy-tailed characteristics and provide a few mechanisms for controlling how heavy the tails should be.

Given such a generator, we must be able to find a reliable gradient to train it.
Heavy-tailed distributions introduce challenges to learning. 
In section \ref{sec:loss}, we show that common metrics, such as Wasserstein distance \cite{arjovsky17} and energy distance \cite{bellemare17} are infinite between sufficiently heavy-tailed distributions.
In these cases, sample gradients do not converge and mini-batch gradient estimates are unstable. 
We propose a solution to this problem where we evaluate the loss function over a metric space on which the distributions are better behaved.

We name the combined approach Pareto GAN.
Pareto GAN uses methods from extreme value theory to estimate the tail index of the marginal input distributions.
It uses this tail index to construct a generator with matching tails and a loss metric that ensures a useful gradient for training. 
We show how Pareto GAN can be used to generate multivariate distributions that have different marginal tail indexes, which suggests a high degree of flexibility in future applications.



%% file: preliminaries.tex
\section{PRELIMINARIES}
\subsection{Generative adversarial networks}
A GAN consists of a generator and a discriminative loss function. The generator is represented as a neural network $f$ that transforms a random variable $Z$ with a known distribution (e.g., uniform, normal) into a new random variable in some output space:

\begin{equation}
    X = f(Z)
\end{equation}

The generator network $f$ is trained to minimize a loss function that discriminates between samples from two distributions. 
The ideal training outcome is that $X$ matches a particular target distribution over the output space and the loss function cannot discriminate between the two distributions.
Popular GAN loss functions include Wasserstein distance \cite{arjovsky17} and maximum mean discrepancy (MMD) \cite{gretton12}\cite{li2017mmd}. 
This paper focuses on Wasserstein distance and energy distance \cite{sejdinovic13}, which is a type of MMD loss function (see section \ref{sec:loss}). 

\subsection{Tail distributions and extreme value theory}



%

\begin{definition}\label{def:excess}
    Let $X$ be a random variable. Define $F(x) = P(X \le x)$ as the cumulative distribution function (CDF) of X.  Define $\bar F(x) = 1-F(x)$ as the complementary cumulative distribution function (CCDF). For random variable X, the conditional excess distribution function is defined 

    \begin{equation}
        \begin{split}
            F_u(y) &= P(X-u \leq y | X>u) \\
            &= \frac{F(u+y)- F(u)}{1 - F(u)}
        \end{split}
    \end{equation}
\end{definition}

\begin{definition}\label{def:gpd}
    The generalized Pareto distribution (GPD), parameterized by tail index $\xi \in \mathbb{R}$ and scaling parameter $\sigma \in \mathbb{R}$, has the following CCDF, which is defined over $\mathbb{R}_+$:

    \begin{equation}
        S(z; \xi,\sigma) = \begin{cases}
            (1 + \xi z/\sigma)^{-\frac{1}{\xi}},      & \text{for $\xi \neq 0$} \\
            e^{-z/\sigma},                            & \text{for $\xi = 0$}.
        \end{cases}
    \end{equation}
\end{definition}


%
The Pickands–Balkema–de Haan theorem \cite{balkema74} states that the conditional excess of a broad class of distributions converge to the GPD as $u \rightarrow \infty$. These distributions include bounded distributions, exponential family distributions (e.g., Gaussian, Laplacian), stable distributions (e.g., Cauchy, Levy), and power law distributions (Student-t, Pareto).  
There are a variety of definitions in the literature for what constitutes a "heavy-tailed" distribution, but we will use the term to denote distributions with a tail index $\xi>0$.






\section{RELATED WORK}

Works using GANs on heavy-tailed data \cite{lin2019generating}\cite{Wiese2019QuantGD} often train on logarithmically transformed data, and exponentiate the GAN output to get back to the original data domain. 
While this can help the learning process, the learned distribution does not meet our definition of heavy tailed, as we will show in section \ref{sec:tails}. 
Other works have used heavy tailed input distributions on bounded domains (e.g. images) \cite{Sun2018StudentsTA}\cite{Upadhyay2019RobustSG}. 
These works focus on representations with non-Gaussian characteristics, but are not concerned with the tails of the output domain (since it is bounded). 
\cite{Wiese2019Copula} presents a proof that a generator network cannot make the tails of its input distribution heavier. 
Our work mirrors some of the arguments in \cite{Wiese2019Copula}, but presents a viable solution to the problem in the Pareto GAN.
Concurrent work \cite{feder2020nonlinear} uses a Student-t prior to produce unbounded heavy-tailed data, which has similar tail characteristics to our GPD prior. 
However, their choices for tail index and loss function were chosen through trial and error. 
We present an approach for choosing these parameters grounded in theory and existing extreme value literature.

%% file: tails.tex
\section{THE ASYMPTOTIC BEHAVIOR OF GAN GENERATORS}\label{sec:tails}
We now examine the asymptotic behavior of GAN generators. To do so, we draw heavily on a property of most neural networks (including all those which we will consider): Lipschitz continuity. Roughly speaking, a Lipschitz continuous function has bounded slope; for brevity, we place a full definition of Lipschitz continuity in the appendix.


\subsection{Generators with bounded support}

\begin{proposition}\label{prop:bounded}
    Let $Z_A$ be a random variable in metric space $(\mathcal{Z},d_\mathcal{Z})$. Let $f: \mathcal{Z} \to \mathcal{X}$ be a Lipschitz continuous neural network with respect to metrics $d_\mathcal{Z}$ and $d_\mathcal{X}$. If $Z_A$ lies within ball of radius $c$ centered around $z_0$, $B_c[z_0] \subseteq \mathcal{Z}$, with probability 1, then there exists a ball $B_d[x_0] \subseteq \mathcal{X}$ such that $P(f(Z_A) \in B_d(x_0)) = 1$.






\end{proposition}



In short, a Lipschitz continuous function always maps a bounded distribution to another bounded distribution. 
Generators with a standard uniform input distribution, therefore, must be bounded.
So are generators with a bounded intermediate layer such as a tanh or sigmoid activation, since such a layer produces a bounded random variable that serves as input to the rest of the network.
While a generator of this type can potentially fit an arbitrary training set (regardless of the distribution that generated it), the probability of producing a sample outside of $B_{d}[x_0]$ is exactly zero. 

The left side of Figure \ref{fig:1} illustrates this phenomenon. 
We trained a GAN with uniform input noise on a heavy tailed data set, namely samples from a mixture of two Cauchy distributions.  
The GAN matches the modes of the distribution somewhat, but as predicted, the tails have a hard cutoff: every sample we generated was between +/- 15.  





\begin{figure*}
  \centering
   \includegraphics{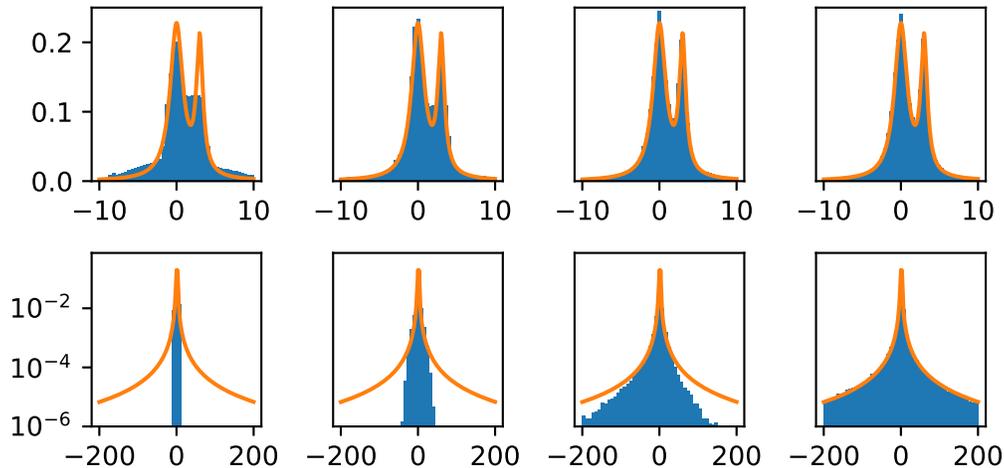}
  \caption{Probability densities of generators with (from left to right) bounded, normal, lognormal and Pareto tails. Generators are trained on a Cauchy mixture with density shown in orange. The top row shows the center of the distribution on a linear scale, while the bottom row shows the tails on a log scale.}\label{fig:1}
\end{figure*}

\subsection{Generators with zero tail index marginals}
A generator that would seem to address this problem combines an unbounded input distribution with an unbounded neural network, such as $X_N$ defined below.

\begin{definition}\label{def:pwl} 
    Let $f_{PWL}: \mathbb{R}^n \rightarrow \mathbb{R}$ be a piecewise linear (PWL) function with a finite number of linear regions.
\end{definition}

\begin{remark}\label{remark:pwl}
    $f_{PWL}$ is Lipschitz continuous with respect to Minkowski distances (metrics which generate the p-norms).
\end{remark}

Definition \ref{def:pwl} encompasses a broad class of neural networks (see, e.g., \cite[Theorem 2.1]{arora16}). 
PWL functions are closed under composition, so it is easy to show that a neural network composed of operations such as ReLUs, leaky ReLUs, max pooling, maxout activation, linear layers, concatenation, addition, and batch normalization (in "test" mode) all meet the requirements of Definition \ref{def:pwl}.




\begin{definition}
    Let $N(\mu,\sigma)$ be a normal distribution. A normal generator $X_N$ is

    \begin{equation}
        X_N = f_{PWL}(Z_N),\, Z_N \sim N(0,1).
    \end{equation}


\end{definition}

Note that $f_{PWL}$ is univariate, and that an arbitrary neural network can be broken up into a set of univariate functions like $f_{PWL}$. In that setting, the distribution of $X_N$ represent a marginal distribution of the output. While it is possible to construct $f_{PWL}$ in such a way that $X_N$ is bounded, in general, $X_N$ has support across the whole real line. However, we now show that $X_N$ has Gaussian tails. 

\begin{theorem}\label{theorem:normal}
    Let $F_u(x)$ be the conditional excess distribution of $X_N$. If $X_N$ is not bounded above, then $F_u(x)$ converges to the normal conditional excess distribution as $u \rightarrow \infty$, i.e., $F(x)$ is a member of the Gumbel domain of attraction.
\end{theorem}

We put the formal proof in the appendix, but outline the intuitions here. Because $f_{PWL}$ has a finite number of convex linear regions, its asymptotic behavior is therefore linear. Moving along any line in the input space eventually enters a "final" linear region, and $f_{PWL}$ acts linearly on all points beyond this threshold. The tails of the input distribution, therefore, are scaled and shifted by $f_{PWL}$, but they retain the shape of their original distribution. Multiple regions of the input space may map to a single output region so the output tail acts like a mixture of Gaussians, which asymptotically behaves like a single Gaussian.

A practice commonly used in GAN literature with heavy tailed data is to exponentiate the output of a generator such as $X_N$:
\begin{equation}
    X_{LN} = \exp{(f_{PWL}(Z_N)-1)},\, Z_N \sim N(0,1).
\end{equation}

Since the tails of $X_N$ act like a Gaussian, the tails of $X_{LN}$ follow lognormal asymptotics. This is a significant improvement in practice, but it still produces a distribution with a tail index of zero. The center two columns of Figure \ref{fig:1} capture these predicted behaviors. In the Gaussian case, the exponential decay is clearly evident in the poor tail approximation. The lognormal generator fares better, but undercounts extreme events. 





\subsection{Pareto generators}\label{sec:pgen}

Ideally, our GAN generator would match some belief we have about the marginal tail behavior. The generators we have discussed so far are not able to capture this type of belief for heavy-tailed distributions. We address this shortcoming with the Pareto GAN generator. In its basic form, a Pareto GAN generator takes a GPD input with tail index $\xi$, which matches the tail index belief from the data.
$\xi$ can be chosen in a variety of ways, such as a tail index estimator (e.g. Hill's estimator \cite{deheuvels88, resnick97}, a kernel-type estimator \cite{wolf2003}), a prior belief, or estimation during the training process.


\begin{definition}
    Let $Z_{\xi} = (U^{-\xi} - 1)/\xi,\, U \sim Uniform(0, 1)$, which is a GPD random variable with tail index $\xi$ and a CCDF of the form $S(x;\xi,1)$. A Pareto GAN generator $X_{\xi}$ parameterized by tail index $\xi$ is defined
    \begin{equation}
        X_{\xi} = f_{PWL}(Z_{\xi}).
    \end{equation}
\end{definition}

\begin{theorem}\label{theorem:pareto}
    Let $F_u(x)$ be the conditional excess distribution of $X_{\xi}$. If $X_{\xi}$ is not bounded above, then $F_u(x)$ converges to $S(x;\xi,\sigma)$ for some $\sigma \in \mathbb{R}$. 
\end{theorem}









The proof is included in the appendix and follows a similar argument to that of Theorem \ref{theorem:normal}. 
Figure \ref{fig:1} illustrates the effectiveness of the Pareto GAN compared with the other approaches. 
All three trained GANs use the exact same network architecture and are trained using energy distance \cite{sejdinovic13}\footnote{To ensure convergence, we train the Pareto GAN with the 2-root energy distance defined in section \ref{sec:loss}.}. 
With $\xi=1$, the tails of the GPD input noise match the tail index of the Cauchy mixture, but the distributions are very different around the modes. 
The GPD is one-sided with a uniformly decreasing density. 
The Cauchy mixture is two-sided and bimodal. 
The Pareto GAN, however, is able to learn an accurate approximation, both around the modes and in the tails.

We can also define a more general form of the Pareto generator.





\begin{corollary}\label{corr:pareto2}

    Let $X_{\alpha}$ be a Pareto GAN generator with tail index $\alpha$. Let 

    \begin{equation}
        X_{\beta} = sign(X_{\alpha}) |X_{\alpha}|^{\beta},\ \beta>0
    \end{equation}

    Let $F_u(x)$ be the conditional excess distribution of $X_{\beta}$. If $X_{\beta}$ is not bounded above, then $F_u(x)$ converges to $S(x;\alpha\beta,\sigma)$ for some $\sigma \in \mathbb{R}$. 
\end{corollary}

The proof builds on Theorem \ref{theorem:pareto} and is in the appendix. Corollary \ref{corr:pareto2} gives a degree of flexibility in constructing a Pareto GAN generator. Importantly, in a multivariate setting we can choose different values of $\beta$ for each output dimension. This allows us to learn a complex joint distribution over variables with different tail indexes. 
This flexibility suggests that there is a clear path to apply Pareto GAN to a broad class of distributions.

%% file: loss.tex
\section{LEARNING HEAVY TAILED DISTRIBUTIONS WITH GANS}\label{sec:loss}

Regardless of the form of the generator being used, in order to train it, we need a loss function with a reliable gradient.
As discussed in \cite{arjovsky17}, said loss function should provide a non-zero gradient for manifolds with zero-measure intersections.
This rules out f-divergences such as Jensen-Shannon divergence.
We identify two additional properties that guide our search for a loss function.

\begin{itemize}[leftmargin=*]
    \item\emph{Finiteness} The loss function should be finite, with both finite gradients and finite expectations of sample gradients. In particular, our choice of loss function must be well defined as such on the target distribution and all distributions that the GAN can generate. 

    \item\emph{Minimal outlier gradient decay} The gradient of the loss function with respect to an outlier should decay as slowly possible.
\end{itemize}

The Wasserstein 1 distance\footnote{Full definition of Wasserstien distance in appendix} (aka Earth-Mover distance) is a very popular loss function for training GANs \cite{bellemare17, arjovsky17}. The Wasserstein distance between two distributions is well defined when the distributions have finite first moments\footnote{Full definition of moments in appendix}.
However, when this is not the case, convergence is no longer guaranteed.
Energy distance \cite{sejdinovic13} is another metric with similar properties.

\begin{definition}\label{def:energydist}
    The energy distance, $E$ between distributions $P$ and $Q$ with random variables $X$, $X'$, $Y$, and $Y'$, on a metric space $(A, d)$ is
    \begin{equation}
        E(P,Q) = 2\mathbb{E}d(X,Y) - \mathbb{E}d(X,X') - \mathbb{E}d(Y,Y')
    \end{equation}
    which is finite, and well defined when $P$ and $Q$ have finite first moments \cite{sejdinovic13}.
\end{definition}

We observe that by changing the metric on the underlying space we can give our target distributions finite first moments.
There are two ways that we can approach this.  First, we consider bounded metrics.
Under a bounded metric all probability density functions will have finite moments.
For example, 

\begin{definition}\label{def:bounded_metric}
    Let $\alpha > 0$. The bounded Euclidean metric induced by $\alpha$ is
    \begin{equation}
        d_\alpha(x,y) = \frac{||x-y||_2}{\alpha + ||x-y||_2}
    \end{equation}
\end{definition}

\begin{remark}
    For all $x,y$, it holds that $d_\alpha(x,y) < 1$, hence for all PDFs $f$ and values $z_0$, $\int d_\alpha(z,z_0) f(z)dz < \int f(z)dz = 1$.
\end{remark}

While using spaces with a bounded Euclidean metric ensures the finiteness of our loss functions, the produced gradient fails to provide much information about the tails.
Intuitively we note that because distances are bounded above by one, "large" and "very large" distances are essentially impossible to distinguish using $d_\alpha$.
Therefore, the gradient of this distance quickly decays to zero, and the metric is not useful in gradient descent. 
Other GAN loss functions, such as the RBF MMD \cite{gretton12} are also bounded and, as such, have this property. 
Hence, in order to satisfy our second criteria, we instead modify the standard Euclidean notion of distance on $\mathbb{R}$.

\begin{definition}\label{def:root_euclid}
    Let $\gamma > 0$. The $\gamma$ Root-Euclidean distance is
    \begin{equation}
        d_{\gamma}(x,y) = ||x-y||_2^{1/\gamma}
    \end{equation}
\end{definition}

\begin{remark}
    For all $\gamma \geq 1$, $d_{\gamma}$ defines a metric on $\mathbb{R}_+$ as $x^{1/\gamma}$ is a monotonically increasing concave function.
\end{remark}

\begin{remark}
    The closer $\gamma$ is to one, the more similar $d_\gamma$ and the Euclidean distance metrics are.
\end{remark}

Here, it is less obvious that we will get finite first moments.
In fact, it is the case that in order to assure finite moments on a given distribution, we must choose our $\gamma$ to suit it.
The following theorem presents the appropriate bound on $\gamma$.

\begin{theorem}\label{thm:finite_moment}
    Let P be a Generalized Pareto Distribution with tail index $\xi$. For all $\gamma > \xi$, P has a finite first moment on the space $(\mathbb{R}, d_\gamma)$.
\end{theorem}

\begin{proof}
    Let $f$ be the PDF of $P$, and consider the first moment of $f$: 
    \begin{equation}
    	\int d_{\gamma}(z, 0) f(z) dz = \int_{\mathbb{R}_+} z^{1/\gamma} (1 + \xi z)^{-\frac{\xi + 1}{\xi}} dz.
    \end{equation}
    
    By estimation, the following two statements are equivalent.
    \begin{equation}
    	\int_{\mathbb{R}_+} z^{1/\gamma} (1 + \xi z)^{-\frac{\xi + 1}{\xi}} dz  < \infty 
    \end{equation}
    \begin{equation}
    	\int_{\mathbb{R}_+} (1 + z)^{-1 + (1/\gamma - 1/\xi)} dz < \infty 
    \end{equation}
    
    By the power rule for integral convergence we see that this statement is equivalent to
    \begin{equation}
        -1 + (1/\gamma - 1/\xi) < -1
    \end{equation}
    which is equivalent to the statement $\gamma > \xi$.
\end{proof}

With this result, we can provide convergence guarantees for both Wasserstein and Energy distance on heavy tailed distributions, so long as we are using the correct distance metric.

\begin{corollary}\label{corr:finite_distances}
    If GPD distributions $P$ and $Q$ with PDFs $X$ and $Y$ respectively have tail indexes $\xi_p, \xi_q < \gamma$, then the Wasserstein and energy distances between them, $W_1(P,Q)$, and $E(P,Q)$ respectively, on the space $(\mathbb{R}, d_\gamma)$ are finite.
\end{corollary}

%% file: exp.tex
\section{EXPERIMENTS}\label{sec:exp}

\subsection{Approximating Univariate Distributions}


\begin{figure}
    \centering
    \includegraphics{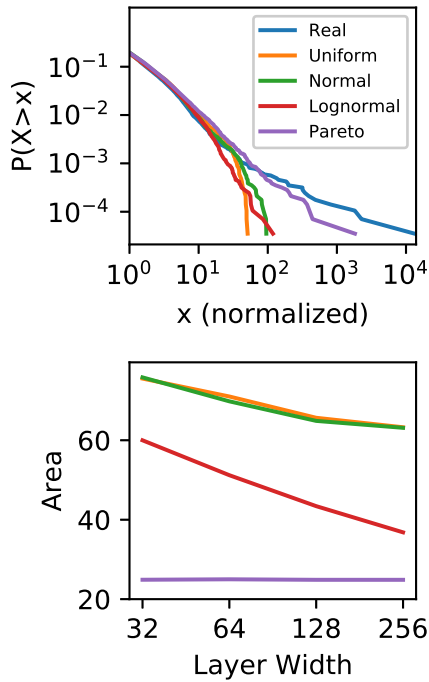}
    \caption{Top: log-log plot for Wiki Traffic data. Bottom: Area of tail errors for different width generators on keystroke data.}\label{fig:2}
\end{figure}

We now demonstrate our GANs on a few different types of data. First we demonstrate our Pareto GANs on a handful of univariate heavy-tailed datasets:

\begin{itemize}
    \item\emph{136 million keystrokes.} This dataset includes inter-arrival times between keystrokes for a variety of users \cite{dhakal2018}.

    \item\emph{Wikipedia Web traffic.} This dataset includes the daily number of daily views of various Wikipedia articles during 2015 and 2016\footnote{https://www.kaggle.com/c/web-traffic-time-series-forecasting}. We train the GANs to reproduce the distribution of view counts.

    \item\emph{SNAP LiveJournal.} This dataset consists of a network graph for the LiveJournal social network \cite{leskovec2008}. We train the GANs to reproduce the distribution of edge counts.

    \item\emph{S\&P 500 Daily Changes.} This dataset consists of the daily prices of the S\&P 500 stocks from 1999 through 2013 \footnote{Downloaded from https://quantquote.com/historical-stock-data, data has been recently removed}. We train the GANs to reproduce the distribution of daily percentage changes in individual stocks.

\end{itemize}

We randomly partition the data into training, validation, and test sets. Training and validation each have a small fraction of the full dataset (<10\%), while the remainder becomes the test set. This allows us to test the ability of the GANs to extrapolate the probabilities events that are more extreme than those in the training data. We normalize all datasets by dividing by the average magnitude of the training set.

Our evaluations consider two metrics. First, we use the Kolmogorov–Smirnov (KS) test statistic between the real and generated samples. The KS statistic is defined as the largest magnitude difference between the CDFs of two distributions, and it is used to test the hypothesis that two sets of samples are from different distributions \cite{hodges58}. KS gives us an indication of how well the modes of the data match, and is independent of any of the loss functions used in training. We use the implementation in scikitlearn \cite{pedregosa2011scikit}. Secondly, we compute the area between the log-log plots of the empirical CCDFs of the real and generated samples. This metric gives us a good indication of how well the generated tails match the real samples. Figure \ref{fig:2} (top) gives an example of such a plot. For $n$ real samples and inverse empirical CCDFs $\bar F_R^{-1}$ and $\bar F_G^{-1}$, the formula is


\begin{equation}\label{eq:area}
    Area = \sum_{i=1}^{n} \left|log \bar F_R^{-1}\left(\frac{i}{n}\right)- log \bar F_G^{-1}\left(\frac{i}{n}\right)\right| log\frac{i+1}{i}.
\end{equation}


We used a common network architecture and training procedure for all experiments. The network consisted of four fully connected layers with 32 hidden units per layer and ReLU activations. For the Pareto GAN, we estimate the tail index from the training data using an open-source implementation\footnote{https://github.com/ivanvoitalov/tail-estimation} of the kernel-type estimator from \cite{wolf2003}. We used a batch size of 256 in all cases. We vary learning rate from $10^{-4}$ to $10^{-6}$ and train for 20,000 iterations. From these networks, we select the model with the best validation loss and evaluate on a test dataset. We use the energy distance loss function from definition \ref{def:energydist} in all cases, but we vary the underlying $d(\cdot,\cdot)$ metric to fit the GAN. For Pareto GAN, we use the $d_{\gamma}(\cdot,\cdot)$ from definition \ref{def:root_euclid}, with $\gamma=2$ on all datasets. This ensured finite loss for all the tail index estimates of all datasets. We used standard Euclidean energy distance for the other GANs. In the lognormal GAN, we follow the common practice of computing the loss function on the log-transformed space \cite{Wiese2019QuantGD}.


\begin{table}[h!]
\centering
\vspace{-5pt}
\caption{Experimental Results}\label{tab:1}
  \begin{tabular}{ l||c|c||c|c} 
    &\multicolumn{2}{|c||}{Keystrokes} &\multicolumn{2}{|c}{Wiki Traffic}\\ 
     
    GAN type          & KS & Area     & KS & Area \\
    \hline
    Uniform   & 0.017 & 67.8 & 0.025 & 10.3 \\
    \hline
    Normal    & 0.020 & 59.5 & 0.023 & 8.6  \\
    \hline
    Lognormal & 0.014 & 41.0 & 0.019 & 9.5  \\
    \hline
    Pareto    & \textbf{0.013} & \textbf{21.1} & \textbf{0.017} & \textbf{4.5} \\
    
    \hline
    \hline
&\multicolumn{2}{|c||}{LiveJournal} &\multicolumn{2}{|c}{S\&P500}\\    
    GAN type          & KS & Area     & KS & Area \\
    \hline

Uniform   & \textbf{0.094} & 15.4  & \textbf{0.011} & 12.6\\
    \hline
    Normal    & 0.103 & 7.8 & 0.019 & 7.3\\
    \hline
    Lognormal & 0.111 & 3.1 & 0.014 & 6.5\\
    \hline
    Pareto    & 0.105 & \textbf{2.0} & 0.062 & \textbf{4.4}\\    
  \end{tabular}
\end{table}

Table \ref{tab:1} compares Pareto GAN to the baseline GANs on the four datasets. In all cases, Pareto GAN provides better tail estimation than other techniques, while generally matching performance on the KS statistic. Figure \ref{fig:2} (top) shows an example of the tail distribution.

Another promising property of the Pareto GAN is that it can learn a more compact representation of heavy tailed datasets than the other models. Other architectures don't naturally produce power-law tails, so they have to use the capacity of the neural network to fit their shallow-tailed distributions into a power-law shape. Since Pareto GAN does produce power-law tails, it can use its neural network capacity to fit the main body of the distribution while still maintaining good tail approximation.

To demonstrate this behavior, we trained neural networks with different layer widths and computed the log-log area metric. Figure \ref{fig:2} (bottom) demonstrates how the Pareto GAN maintains its tail approximation down to width 32, while the other GANs see a sharp drop off in tail accuracies.



     

\subsection{Approximating Multivariate Distributions}

One of the attributes that makes GANs attractive is that they can learn manifolds embedded in high-dimensional spaces. 
In practice, high-dimension data (e.g. images) does not always span the entire space; instead, they are usually confined to a low dimensional manifold. Learning to align manifolds is a hard problem that precludes the use of some loss functions, such as Jensen-Shannon divergence \cite{arjovsky17}. 
Data with heavy tails further complicates things. 
We show in this section that Pareto GAN is capable of learning distributions with all of these characteristics. 

To apply Pareto GAN to multivariate data, we independently estimate the tail index of each dimension, which scales linearly the number of dimensions. 
We construct Pareto GAN with $\xi=1$ input noise and leverage Corollary \ref{corr:pareto2} once for each dimension, setting $\beta$ to the estimated tail index.
This results in a joint distribution approximately matching the tail indexes of each dimension. 
We then train with root-Euclidean energy distance.  We set $\gamma$ to be the largest estimated tail index \textit{plus one}. This ensures that the expected loss is finite, but still emphasizes the tails sufficiently.

We now define some multi-dimensional distributions with heavy-tailed characteristics and attempt to train GANs to approximate them. First, we define a joint distribution $[X_0, X_1]$ with components defined as follows: 

\begin{equation}
    \begin{split}   
        X_0 &= A + B\\
        X_1 &= sign(A - B)|A-B|^{1/2}
    \end{split}
\end{equation}

where $A$ and $B$ are independent Cauchy RVs. Note that $X_0$ and $X_1$ have different tail indexes (1 and 1/2, respectively) and are not independent.

We trained a Pareto GAN on this distribution. 
The results of this process are shown in Figure \ref{fig:2d}.  
As expected, the marginals match closely, and the joint distributions appear to be close as well.

%

\begin{figure*}
    \centering
    \includegraphics{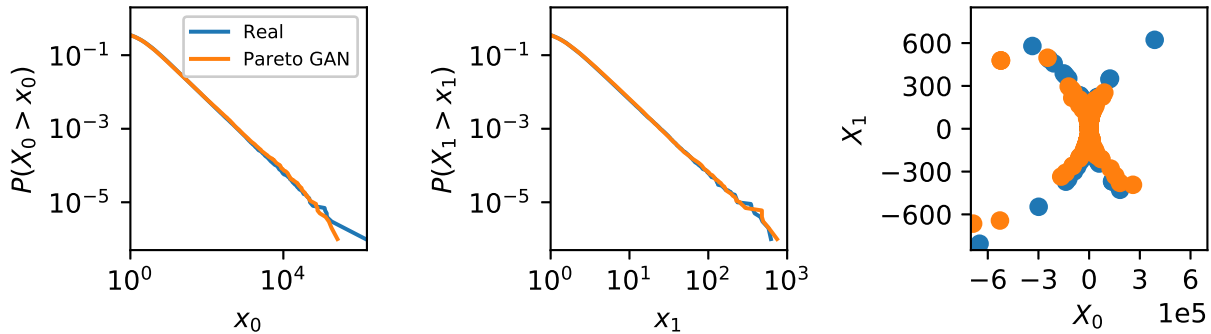}
    \caption{Left/Center: log-log plots of marginal distributions. Right: scatter plot of 1M data samples. }\label{fig:2d}
\end{figure*}


Our second multivariate distribution is a high dimensional manifold. 
We define a $d$-dimensional distribution in which all points lie on a $c$ dimension manifold, with $c\ll d$. Furthermore, we give each dimension a different tail index. The random vector is defined:
\begin{equation}\label{eq:highd}
    X = \power(CY, t)\;
\end{equation}

 $Y\in\mathbb{R}^c$ is a $c$-dimension hidden random variable independently drawn from Cauchy distribution, $C\in \mathbb{R}^{d\times c}$ is a constant matrix for transforming the hidden variables to the observable variables, $t\in \mathbb{R}^d$ is a constant vector containing the target tail index, and $\power$ is the element-wise power operation. 
In the following experiments, we set $c=100$, $d=1000$. The elements in $C$ are independently drawn from $\mathcal{N}(0,1)$, and the elements in $t$ are independently drawn from $\uniform([0.5,3])$.
 
We trained four GAN variants on 10k samples from this distribution. 
We trained "normal" and "lognormal" GANs as in previous experiments, and the Pareto GAN as outlined above. 
In order to examine the importance of our proposed loss function, we also trained Pareto GAN with basic energy distance (i.e., with $\gamma=1$).  
This allows the generator to express heavy tailed distributions, but doesn't guarantee that training will converge since the expected loss function is infinite.  
For all GAN variants, we use 200-dimension input noise to the generator.
The generator network consists of 4 fully connected layers with 256 units on each layer. 
The batch size is 256, and the number of training iterations is 200000. 
We use 10000 samples for training, and a disjoint set of 1000000 samples for evaluation.

For each marginal distribution, we computed the area metric from equation \ref{eq:area} between 1M real and generated samples. 
Since the marginals are two-sided, we compute the area metric for both sides and average them.  
We report the average area metric across all dimensions. 

We also examined how well the generator captures the data manifold based on how close samples are to the true manifold.  
This would be very difficult to do with real data, but our synthetic distribution allows us to compute the distance exactly in a warped version of the space. 
We invert the power transform and project generated samples onto the linear manifold represented by $CY$. 
The distance to the linear manifold is

\begin{equation}
MDist = d(\power^{-1}(\hat{x}, t),\power^{-1}(\hat{x}, t)^TP)
\end{equation}

where $P$ is the projection matrix $C(C^TC)^{-1}C^T$, $\hat{x}$ is a generated sample, and $d$ is Euclidean distance.
We report the mean (natural) log MDist to ensure that our metric has finite expectation for all models. 

We ran these experiments with 3 random seeds (arbitrarily chosen 1000, 1001, 1002).  
The seed impacts the choices of $C$ and $t$, as well as network initialization, training, and sampling.
The three seeds produced similar results. 
We report the average of these three trials in Table \ref{tab:2}


\begin{table}[h!]
\centering
\vspace{-5pt}
\caption{Experimental Results}\label{tab:2}
  \begin{tabular}{ l||c|c} 
    GAN type          & Mean Area & Mean Log MDist \\
   \hline
    Normal  & 132 & 22.1   \\
    \hline
   Lognormal   & 216 & 35.7   \\
   \hline
    Pareto (ED) & 197 & 9.8  \\
    \hline
    Pareto (root-ED)  & \textbf{28}  & \textbf{7.7} \\
  \end{tabular}
\end{table}
 
The root-ED Pareto GAN clearly performs the best on both metrics. 
It is able to match the tails of the marginals fairly well while producing points close to the manifold. 
Interestingly, the ED Pareto GAN also produces points close to the manifold, but its tail estimation is quite bad. 
We investigated the cause of this by looking at a few marginal distributions and observed that the GAN marginals were often bounded on one side\footnote{See Figure \ref{fig:highdmarg} in appendix}. 
We note that Theorem \ref{theorem:pareto}, does not preclude a Pareto GAN from being bounded. 
Instead, the Pareto GAN is failing to learn two-sided tails when we use an unstable energy distance loss function. 
Replacing this with a stable root-Euclidean energy distance allows learning to be successful.

%% file: conclusions.tex
\section{CONCLUSIONS}





In this paper, we have identified a specific bias in the application of GANs to open domains, namely the implicit prior of their tail behavior. 
We have also identified shortcomings of some common loss functions when applied to heavy tailed data. 
Our proposed Pareto GAN addresses both of these shortcomings, providing a way for the generator to express heavy tailed distributions and learn such distributions effectively.

%% file: app.tex
\onecolumn
\appendix

\section{DEFINITIONS}

\addtocounter{definition}{-8} 

\begin{definition}\label{def:lipschitz}
    Given metric spaces $(\mathcal{X}, d_\mathcal{X})$ and $(\mathcal{Y}, d_\mathcal{Y})$, we say a function $f: \mathcal{X} \to \mathcal{Y}$ is Lipschitz continuous if and only if there is a constant $k$ where
    \begin{equation}
        \forall x_1, x_2 \in \mathcal{X}, d_\mathcal{Y}(f(x_1),f(x_2)) \leq k d_\mathcal{X}(x_1,x_2)
    \end{equation}
    If the above equation holds for a particular $k$, we say $k$ is a Lipschitz constant for $f$ and that $f$ is $k$-Lipschitz continuous.
\end{definition}

Roughly speaking, a Lipschitz constant is a bound on the slope of $f$. Lipschitz continuous functions are closed under composition, so a network composed of Lipschitz continuous operations is also Lipschitz continuous. The vast majority of common neural network operations meet this criterion, including fully connected and convolutional layers, pooling layers, and activation functions such as sigmoid, tanh, and ReLU.

\begin{definition}\label{def:Wasserstein}
    The Wasserstein distance, $W_1$ between distributions $P$ and $Q$ on a metric space $(A, d)$ is
    \begin{equation}
        W_1(P,Q) = \inf_{\pi \in \Pi(P,Q)} \int_{A \times A} d(x,y) \pi(x,y)
    \end{equation}
    where $\Pi(P,Q)$ is the set of all joint probability distributions on $A$ with marginals $P$ and $Q$.
\end{definition}

The Wasserstein distance uses the distance measure in the underlying space to consider how much mass must be moved what distance in order to deform one distribution into the other.
However, if two distinct distributions do not have a well defined mean (or infinite mean) then it makes sense that the amount of work necessary to deform one into the other can be infinite.

\begin{definition}\label{def:moment}
    We say that a distribution $P$ with PDF $f$ has a finite $n$'th moment on the metric space $(A, d)$ if
    \begin{equation}
        \int_A d(z,z_0)^n f(z) dz < \infty
    \end{equation}
    for some $z_0 \in A$.
\end{definition}

Note that the first moment is the mean, and the second (when the funciton has the mean subtracted away) is the variance. Moments represent information about a function over its whole domain, and the existance and non-existance of moments tends to provide information about how well behaved a function is.

\section{PROOFS}

\addtocounter{theorem}{-3}
\addtocounter{proposition}{-1}
\addtocounter{corollary}{-2}

\subsection{Proof of Proposition \ref{prop:bounded}}

\input{bounded_proof}
\subsection{Proof of Theorem \ref{theorem:normal}}
\input{normal_proof}
\subsection{Proof of Theorem \ref{theorem:pareto}}
\input{pareto_proof}
\subsection{Proof of Corollary \ref{corr:pareto2}}
\input{pareto2_proof}
\section{ADDITIONAL RESULTS}

Figures \ref{fig:a1} and \ref{fig:a2} show additional results from the experiments run in section \ref{sec:exp}. Figure \ref{fig:a1} shows the tail approximation of the different GANs on each of the four datasets.  
Figure \ref{fig:a2} shows how the size of the neural network affects its ability to model the tails of the data on three different datasets.  Pareto GAN can accurately model the tails with very small networks, while the other generators need more significant network capacity to do so.

\begin{figure*}
    \centering
        \includegraphics{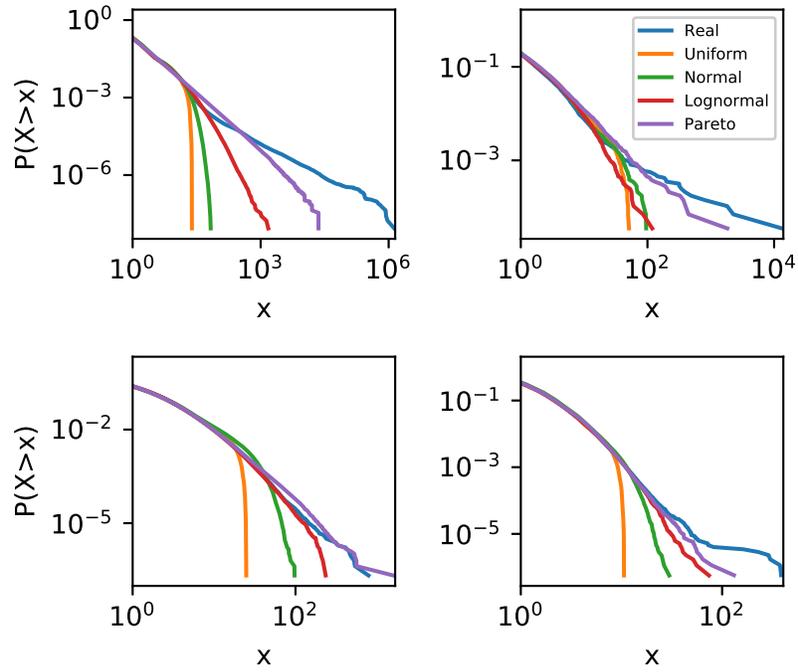}
    \caption{Log-Log plot of the CCDFs for the Keystroke (upper left), Wiki Traffic (upper right), LiveJournal (lower left), and S\&P 500 (lower right) datasets}\label{fig:a1}
\end{figure*}

\begin{figure*}
    \centering
        \includegraphics{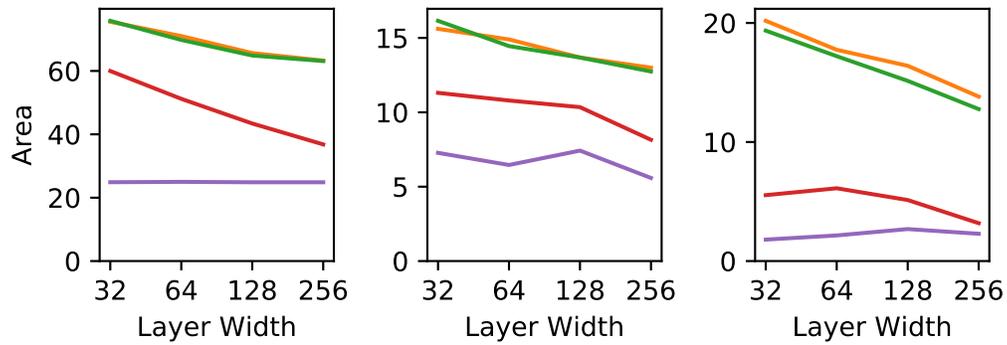}
    \caption{Area vs. layer width for the Keystroke (left), Wiki Traffic (center), LiveJournal (right) datasets}\label{fig:a2}
\end{figure*}


Figure \ref{fig:highdmarg} shows an example of the positive and negative sides of the first marginal (X0) of the 1000-dimensional distribution defined in equation \ref{eq:highd} and used to produce the results in Table \ref{tab:2}. This plot is from seed 1000. The Pareto GAN trained with ED learns a one-sided (all negative) marginal distribution, even though the tail index estimate is fairly good. Training with root-Euclidean ED allows for successful, stable optimization.

\begin{figure*}
    \centering
        \includegraphics{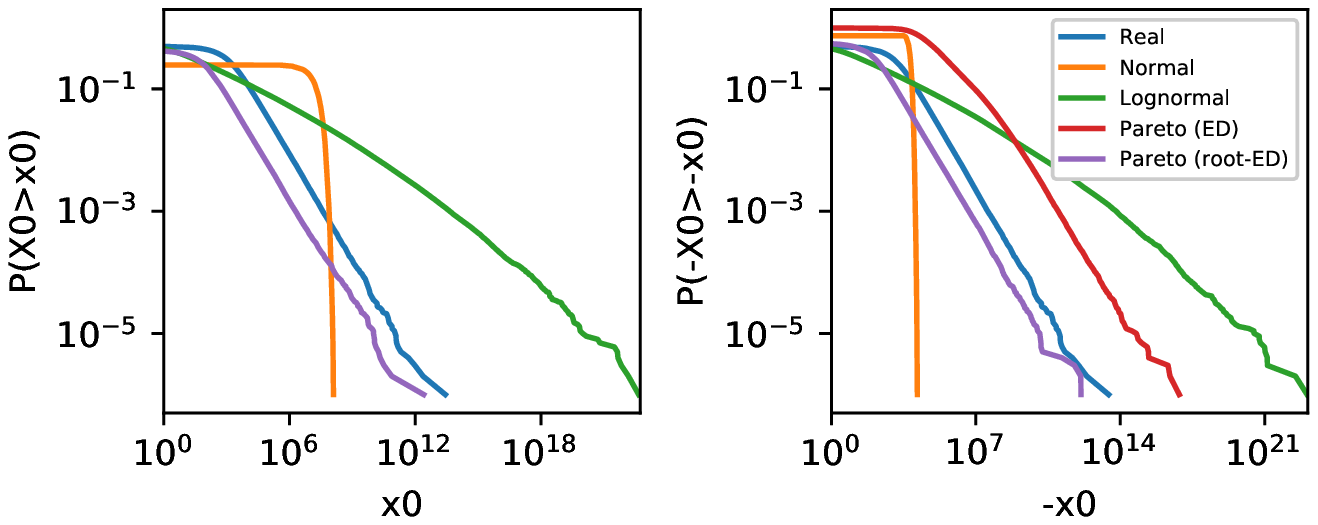}
    \caption{Log-log plots of the positive (left) and negative (right) tails of the first marginal of the real 1000-dimensional distribution and the different GANs.}\label{fig:highdmarg}
\end{figure*}

%% file: bounded_proof.tex
\begin{proposition}
Let $Z_A$ be a random variable in metric space $(\mathcal{Z},d_\mathcal{Z})$. Let $f: \mathcal{Z} \to \mathcal{X}$ be a Lipschitz continuous neural network with respect to metrics $d_\mathcal{Z}$ and $d_\mathcal{X}$. If $Z_A$ lies within ball of radius $c$ centered around $z_0$, $B_c[z_0] \subseteq \mathcal{Z}$, with probability 1, then there exists a ball $B_d[x_0] \subseteq \mathcal{X}$ such that $P(f(Z_A) \in B_d[x_0]) = 1$.
\end{proposition}

\begin{proof}\label{proof:bounded}

From the definition of Lipschitz continuity, there must exist some $k$ where
\begin{equation}
\forall z \in \mathcal{Z}, d_\mathcal{X}(f(z_0),f(z)) \le k d_\mathcal{Z}(z_0,z)
\end{equation}
From the definition of a ball, 
\begin{equation}
\forall z \in B_c[z_0], d_\mathcal{Z}(z_0,z) \leq c.
\end{equation}

\begin{equation}
 \begin{split}
  P(Z_A \in &B_c[z_0]) = 1 \\
  & \Rightarrow P(d_\mathcal{Z}(z_0,Z_A) \leq c) = 1\\
  & \Rightarrow P(d_\mathcal{X}(f(z_0),f(Z_A))\leq kc) = 1\\
  & \Rightarrow P(f(Z_A) \in B_{kc}[f(z_0)]) = 1
 \end{split}
\end{equation}

\end{proof}

%% file: normal_proof.tex
\begin{theorem}
Let $F_u(x)$ be the conditional excess distribution of $X_N$. If $X_N$ is not bounded above, then $F_u(x)$ converges to the normal conditional excess distribution as $u \rightarrow \infty$.
\end{theorem}

\begin{proof} 
    Let $k$ be a Lipschitz constant of $f_{PWL}$. Let $\hat{f}(z) = f_{PWL}(z) - f_{PWL}(0)$. Note that subtracting a bias does not change the Lipschitz constants of a function, and that $\hat{f}(z)$ still meets definition \ref{def:pwl}. It suffices to show that

    From definition \ref{def:lipschitz} it is clear that for any $c$,
    \begin{equation}
        ||\hat{f}(z)|| > c \Rightarrow k ||z|| > c.
    \end{equation}

    Since $\hat{f}$ has a finite number of convex linear regions, there exist positive number $c$ and real values $w_1, w_2, b_1, b_2$ such that $\hat{f}(z)=w_1z+b_1$ for $z>c$ and $\hat{f}(z)=w_2z+b_2$ for $z<-c$. In this case, the conditional excess distribution of $\hat{f}(Z)$ is

    \begin{equation}
        \begin{split}
            F_u(x) = &[w_1>0] \frac{P\left(Z\le \frac{x+u-b_1}{w_1}\right)   -  P\left(Z\le \frac{u-b_1}{w_1}\right)     }{P\left(Z> \frac{u-b_1}{w_1}\right)} \\
            &+ [w_2<0] \frac{P\left(Z\le \frac{x+u-b_2}{w_2}\right) - P\left(Z\le \frac{u-b_2}{w_2}\right)}{P\left(Z> \frac{u-b_2}{w_2}\right)}
        \end{split}
    \end{equation}

    \begin{equation}
        \begin{split}
            F_u(x) = &[w_1>0]\frac{\Phi\left(\frac{x+u-b_1}{w_1}\right) - \Phi\left(\frac{u-b_1}{w_1}\right)}{1-\Phi\left(\frac{u-b_1}{w_1}\right)} \\
            &+ [w_2<0]\frac{\Phi\left(\frac{x+u-b_2}{w_2}\right) - \Phi\left(\frac{u-b_2}{w_2}\right)}{1-\Phi\left(\frac{u-b_2}{w_2}\right)}
        \end{split}
    \end{equation}



    where $\Phi$ is the normal CDF and square brackets evaluate to one if the condition is met and zero otherwise. Now consider the conditions $w_1>0$ and $w_2<0$. If both are false, then the distribution is bounded. If one of the conditions is true, then $F_u(x)$ is exactly the conditional excess distribution of a normal random variable. If both conditions are true, $F_u(x)$ has the tail of a Gaussian mixture. As $u \rightarrow \infty$, $F_u(x)$ is dominated by the component with the larger weight (or the larger bias if the weight are identical), thus converging to a normal conditional excess distribution.
\end{proof}

%% file: pareto_proof.tex
\begin{theorem}
    Let $F_u(x)$ be the conditional excess distribution of $X_{\xi} = f_{PWL}(Z_{\xi})$. If $X_{\xi}$ is not bounded above, then $F_u(x)$ converges to $S(x;\xi,\sigma)$ for some $\sigma \in \mathbb{R}$. 
\end{theorem}

\begin{proof} 
    The proof largely follows from the the proof for Theorem \ref{theorem:normal}. Let $\hat{f}(z) = f_{PWL}(z) - f_{PWL}(0)$. There exist positive number $c$ and real values $w_1, w_2, b_1, b_2$ such that $f_{PWL}(z)=w_1z+b_1$ for $z>c$ and $f_{PWL}(z)=w_2z+b_2$ for $z<-c$. As in Theorem \ref{theorem:normal}, if $w_1\leq0$ and $w_2\leq0$ then $X_{\xi}$ is bounded above.

    \begin{definition}\label{def:amom}
        Define the asymptotic moments of a random variable
        \begin{equation}
            m_{\gamma}(X)=\lim_{t \rightarrow \infty} E\left[ \left(\frac{X}{t}\right)^{\gamma} \Big\vert X > t \right].
        \end{equation}
    \end{definition}

    From Theorem 8(a) in \cite{balkema74}, it suffices to show that for $\gamma>0$

    \begin{equation}
        \gamma<\frac{1}{\xi} \iff m_{\gamma}(X_{\xi}) \text{ exists and is finite.}
    \end{equation}

    The behavior of a random variable over a finite region (e.g., $Z_{\xi}<c$) does not affect which asymptotic moments are finite, nor does scaling and shifting. For a mixture of random variables, an asymptotic moment $m_{\alpha}(X)$ is finite if and only if $m_{\gamma}(X_i)$ is finite for each constituent variable $X_i$. 

    For $Z_{\xi}>c$, $X_{\xi}$ is a mixture of scaled and shifted copies of $Z_{\xi}$. Therefore $X_{\xi}$ has the same finite moments as $Z_{\xi}$, and therefore its conditional excess distribution converges to $S(x;\xi, 1)$.

\end{proof}

%% file: pareto2_proof.tex
\begin{corollary}
    Let $X_{\alpha}$ be a Pareto GAN generator with tail index $\alpha$. Let 

    \begin{equation}
        X_{\beta} = sign(X_{\alpha}) |X_{\alpha}|^{\beta}, \beta>0
    \end{equation}

    Let $F_u(x)$ be the conditional excess distribution of $X_{\beta}$. If $X_{\beta}$ is not bounded above, then $F_u(x)$ converges to $S(x;\alpha\beta,\sigma)$ for some $\sigma \in \mathbb{R}$. 
\end{corollary}

\begin{proof}
    Considering the right tail, we can ignore the negative case and simply use $X_{\alpha}^{\beta}$. From theorem \ref{theorem:pareto}, $X_{\alpha}$ is bounded or converges to a GPD. In the bounded case, raising $X_{\alpha}$ to $\beta$ still produces a bounded variable. As in Theorem \ref{theorem:pareto}, in the unbounded case  it suffices to show that for $\gamma>0$

    \begin{equation}
        \gamma<\frac{1}{\alpha\beta} \iff m_{\gamma}(X_{\beta}) \text{ exists and is finite.}
    \end{equation}

    From Definition \ref{def:amom},

    \begin{equation}
        \begin{split}
            m_{\gamma}(X_{\beta})= & \lim_{t \rightarrow \infty} E\left[ \left(\frac{X_{\beta}}{t}\right)^{\gamma} \Big\vert X_{\beta} > t \right]\\
            = & \lim_{t \rightarrow \infty} E\left[ \left(\frac{X_{\alpha}^{\beta}}{t}\right)^{\gamma} \Big\vert X_{\alpha}^{\beta} > t \right]\\
            = & \lim_{t \rightarrow \infty} E\left[ \left(\frac{X_{\alpha}}{t}\right)^{\gamma\beta} \Big\vert X_{\alpha} > t \right]\\
            = & m_{\gamma\beta}(X_{\alpha})
        \end{split}
    \end{equation}

    Therefore, $m_{\gamma}(X_{\beta})$ is finite if and only if $m_{\gamma\beta}(X_{\alpha})$ is finite, which is true if and only if $\gamma<\frac{1}{\alpha\beta}$.











\end{proof}